%% file: main.tex
\documentclass{article}
\usepackage{array}
\usepackage{multirow}
\usepackage{makecell}
\usepackage{relsize}
\usepackage{physics}
\usepackage{amsfonts}
\usepackage{amsmath,amssymb}
\usepackage{amsthm}
\usepackage{natbib}
\usepackage{enumitem}
\usepackage{mathrsfs}
\usepackage{listings}
\usepackage[english]{babel}

\usepackage{tikz}
\usetikzlibrary{decorations.pathreplacing}
\usepackage{bbm}

\def\argmin{\mathop{\!\arg\!\min}}

\newcommand{\M}{\mathcal{M}_n^\uparrow}

\usepackage[english]{babel}
\usepackage{booktabs}
\usepackage{nicematrix}

\newcommand{\ind}[1]{{{\mathbf{1}_{\left\{#1\right\}}}}}
\newcommand{\reals}{{\mathbb{R}}}
\newcommand{\R}{\reals}
\usepackage[letterpaper,top=2cm,bottom=2cm,left=2.5cm,right=2.5cm,marginparwidth=1.75cm]{geometry}

\usepackage{amsmath}
\usepackage{graphicx}
\usepackage[colorlinks=true, allcolors=blue]{hyperref}
\usepackage{bbm}
\usepackage{tikz}

\definecolor{plum}  {rgb}{.4,0,.4}
\definecolor{forest}  {rgb}{0,.6,0}
\definecolor{midnight}  {rgb}{0,0,.5}

\usepackage{bm}
\usepackage{subcaption}
\usepackage{graphicx}
\usepackage[stable]{footmisc}
\usepackage{setspace}
\usepackage[toc,page]{appendix}

\usepackage{titlesec}
\titleformat{\paragraph}[runin]
{\normalfont\bfseries}{}{1em}{}[:]
\titlespacing*{\paragraph}{0pt}{6pt}{3pt}

\usepackage[capitalise]{cleveref}
\crefformat{appendix}{\S#2#1#3}

\newtheorem{theorem}{Theorem}
\theoremstyle{definition}

\newtheorem{lemma}[theorem]{Lemma}
\newtheorem{proposition}[theorem]{Proposition}

\newtheorem{remark}{Remark}

\usepackage{xspace}

\usepackage{xcolor}

\newcommand{\EE}{\mathbb{E}}
\newcommand{\PP}{\mathbb{P}}

\usepackage{mathtools} % for \mathclap if you want it later
\usepackage{xcolor}

\newif\iffinal
\iffinal
    \newcommand{\todo}[1]{}
    \newcommand{\raphael}[1]{}
\else
    \newcommand{\todo}[1]{{\textcolor{red}{[to do: #1]}}}
    \newcommand{\raphael}[1]{{\color{cyan}Raphael: #1}}
\fi

\title{An analysis of binary isotonic regression: \\ degrees of freedom and implications for calibration}

\usepackage[blocks, affil-it]{authblk}

\author[1]{Raphael Rossellini}
\author[1]{Rina Foygel Barber}
\author[2]{Zhimei Ren}
\author[3]{Jake A. Soloff}

\affil[1]{Department of Statistics, University of Chicago} 
\affil[2]{Department of Statistics and Data Science, University of Pennsylvania}
\affil[3]{Department of Statistics, University of Michigan}

\date{\today}

\begin{document}

\maketitle

\begin{abstract}
\input{src/abstract}
\end{abstract}

\input{src/introduction}

\input{src/dof_result}

\input{src/calibration_consequences}

\input{src/discussion}

\input{src/acknowledgement}

\bibliographystyle{apalike}
\bibliography{refs.bib}

\appendix

\input{src/appendix}

\input{src/error_term_certificate}

\end{document}

%% file: src/abstract.tex
    Isotonic regression is a canonical tool for estimating monotone functions and calibrating probabilistic predictors. We provide a fully sharp finite-sample characterization of its worst-case degrees of freedom on binary samples. Specifically, we identify the binary sequences that maximize the number of distinct fitted values produced by isotonic regression. We develop a sharp bound on the degrees of freedom with a leading term of $\frac{3}{(4\pi^2)^{1/3}} n^{2/3}$ using analytic number theory, improving on previous bounds.

We then apply this result to calibration. Calibration is a central requirement for probabilistic prediction, and isotonic regression is a widely used post-processing method for improving calibration.  Building on deterministic degrees-of-freedom bounds, we derive, to our knowledge, the first nontrivial distribution-free guarantee on the Expected Calibration Error (ECE) of isotonic regression. This ECE bound is fully model-free and distribution-free, only assuming $Y \in \{0,1\}$.

%% file: src/introduction.tex
\section{Introduction}

Isotonic regression is a widely used method for fitting a monotone (non-decreasing) function to data. An important special case is when $Y$ is a binary variable, in contexts such as heterogeneous treatment effect estimation, modeling disease risk, and calibrating binary classifiers \citep{van2023causal, repec:bep:mchbio:1037, zadrozny2002transforming}. 

A key quantity for theoretical analyses of isotonic regression on finite data is the \emph{degrees of freedom}: the number of unique fitted values produced. A small number of fitted levels acts as a complexity measure for isotonic regression: even though the estimator is fit nonparametrically, its output may take on far fewer distinct values than the sample size. Thus, bounds on the degrees of freedom may make finite-sample guarantees possible without imposing smoothness, noise, or well-specification assumptions.

However, isotonic regression in general may have degrees of freedom that is exactly $n$, if we do not place additional assumptions on the properties of the data. (For example, suppose $Y=f(X)$, where $f: \mathcal X \to \mathcal Y$ is a strictly increasing function and $X$ is drawn from a continuous distribution; then fitting isotonic regression to $\{(X_i, Y_i) \}_{i=1}^n$ will yield $\hat f(X_i) = Y_i$ for each $i = 1,\dots,n$.) To date, most sublinear bounds on the degrees of freedom for isotonic regression have proceeded by placing assumptions on the presence of additive noise or the lack of signal. For example, \citet{meyer2000degrees} assume that there is additive noise which is homoskedastic and Gaussian with a fixed scale, and establish that the degrees of freedom is in expectation $O(n^{1/3})$. In contrast, \citet{andersen1954fluctuations} demonstrates that the expected degrees of freedom is $O(\log n)$ when there is no signal (i.e., $f(x) \equiv 0$). 

Recently, \citet{dimitriadis2023honest} established a result of a very different flavor, proving that the degrees of freedom for isotonic regression is upper-bounded by $3n^{2/3}$ under the single assumption that $Y$ is binary. This result requires no assumptions on the distribution of $Y$ or the relationship between $Y$ and $X$. Unlike previous bounds, theirs is fully deterministic, holding for any realization of $\{ (X_i, Y_i)\}_{i=1}^n$.

Our first contribution sharpens this binary $Y$ result to its optimal form. Leveraging tools from analytic number theory, we demonstrate that our construction yields an upper bound with a leading constant of $\frac{3}{(4\pi^2)^{1/3}}$, which improves upon the leading constant of 3 in \citet{dimitriadis2023honest} and identifies the best leading constant possible.

Our second contribution applies this bound on the number of fitted levels to calibration of probabilistic classifiers. Isotonic regression is widely used as a post-processing method for binary forecasts, but questions have remained over the possibility of out-of-sample calibration guarantees. Our deterministic bound on the number of fitted levels gives a direct route to such a guarantee: for an independent test point, isotonic calibration has Expected Calibration Error that is
$
O_p(n^{-1/6}(\log n)^{1/2}),
$ 
under no assumptions on the joint distribution of $(X,Y)$ beyond $Y\in\{0,1\}$.  To our knowledge, this is the first non-trivial distribution-free bound on the expected calibration error of isotonic regression.\footnote{The authors thank Aaditya Ramdas for making them aware of this open problem.}

%% file: src/dof_result.tex
\section{Degrees of freedom result}

We study the isotonic least squares estimator
\[
\mathrm{iso}(y) := \argmin_{\theta\in \M}\|y-\theta\|^2_2,
\]
where $\|\cdot\|$ denotes the Euclidean norm, and $\M:= \{\theta\in \mathbb{R}^n : \theta_1\leq \dots \leq \theta_n\}$ is the monotone cone of non-decreasing sequences of length~$n$.

Our main result uniformly bounds the number of unique values $\hat{k}(y)$ of $\mathrm{iso}(y)$ over all length $n$ binary sequences $y\in \{0,1\}^n$.

\begin{theorem}\label{thm:main}
    The number of unique fitted values $\hat k(y)$ for $\textnormal{iso}(y)$ satisfies
\begin{align}\label{eq:thm_main_upper}
         \sup_{y \in \{0,1 \}^n}
        \hat{k}(y) = \frac{3}{(4\pi^2)^{1/3}} n^{2/3} + \delta_n,
    \end{align}
    where $|\delta_n| = O(n^{1/3} \log n)$ and is specified exactly in the proof via \cref{eqn:tight_df_bound}.
\end{theorem}

The result of Theorem~\ref{thm:main} applies to isotonic regression when the ordering constraint requires that $\theta_i\leq \theta_j$ for each $i\leq j$ (with $i,j\in\{1,\dots,n\}$). In many contexts, it is common to consider variants of isotonic regression that require different types of ordering constraints, and our results apply to such settings as well.
To make this concrete, define 
\[\mathcal{M}_n^{\uparrow,\preceq}:= \{\theta\in \mathbb{R}^n : \theta_i\leq\theta_j\textnormal{ for all }i\preceq j\},\]
where $\preceq$ is any \emph{preordering} on $\{1,\dots,n\}$, meaning that $\preceq$ is transitive, and satisfies $i\preceq i$ for all $i$. Allowing $\preceq$ to be a preordering includes partial orderings and total orderings as special cases, but also allows for ties. This is important in statistical settings: for instance, if we observe data points $(X_i,Y_i)\in\mathbb{R}\times\{0,1\}$ and would like to perform isotonic regression of $Y$ onto $X$, then we would define $i\preceq j$ whenever $X_i\leq X_j$ (and we may have a tie, $i\preceq j$ and $j\preceq i$, if $X_i=X_j$).

\begin{remark}\label{rmk:preordering}
The bound established in the main theorem applies in the more general setting of a preordering, as well. If we define $\hat{k}_{\preceq}(y)$ as the number of unique values in $\textnormal{iso}_{\preceq}(y)$, where
$
\textnormal{iso}_{\preceq}(y) := \argmin_{\theta\in \mathcal{M}_n^{\uparrow,\preceq}}\|y-\theta\|^2_2,
$
then it holds that
\begin{equation}\label{eqn:result_preordering}\sup_{y \in \{0,1 \}^n}
        \hat{k}_{\preceq}(y) \leq \frac{3}{(4\pi^2)^{1/3}} n^{2/3} + \delta_n\end{equation}
(following the same proof). However, this may be a strict inequality, unless $\preceq$ is a total ordering.
\end{remark}

\begin{remark}\label{rmk:levelset} 
The key property of isotonic regression used in the proof is the following property: it partitions the indices $i=1,\dots,n$ into blocks, such that within any block, the value $\mathrm{iso}(y)_i$ is equal to the sample mean of the $y$ values within the block  \citep{lee1983min}. This type of proof holds also for any mapping $y\mapsto T(y)\in\R^n$ that satisfies the following condition:
\begin{equation}\label{eqn:local_avg}\textnormal{(Level-set averaging) }
    \textnormal{For any $y\in\R^n$ and any $i\in[n]$, \ }T(y)_i = \frac{\sum_{j=1}^n \ind{T(y)_j=T(y)_i} \cdot y_j}{\sum_{j=1}^n \ind{T(y)_j=T(y)_i}}.
\end{equation}
That is, the right-hand side of \cref{eq:thm_main_upper} also holds as an upper bound for any level-set averaging algorithm, including histogram binning \citep{ zadrozny2001obtaining} and unimodal regression \citep{frisen1986unimodal}. 
\end{remark}

\begin{remark}
    While we only prove that $|\delta_n| = O(n^{1/3}\log n)$ in \cref{thm:main}, we hypothesize that $|\delta_n| = \widetilde{O}(n^{1/6})$. Assuming the Riemann hypothesis allows tighter bounds on partial sums of the Möbius function \citep{Soundararajan+2009+141+152}, which are related to partial sums of the Euler-totient function (see \cref{lemma:euler_totient_moment}).  We numerically verify $\delta_n \leq 0.63n^{1/6}$ for $n\le 2\cdot 10^{23}$ in \cref{prop:error_term_certificate}.
\end{remark}

\subsection{Proof of Theorem~\ref{thm:main} (and Remarks~\ref{rmk:preordering} and~\ref{rmk:levelset})}
The proof of Theorem~\ref{thm:main} is constructive. We demonstrate the worst-case binary sequence $y$ yields a sorted subset of the Farey series \citep[Section 3.1]{hardy1979introduction} as the fitted values. We then use a standard calculation in analytic number theory for the first moment of Euler's totient function~$\varphi$. 
As mentioned above, this proof also verifies that the right-hand side of~\eqref{eq:thm_main_upper} provides an upper bound more generally, in the settings of Remarks~\ref{rmk:preordering} and~\ref{rmk:levelset}.

\begin{proof}
    Since isotonic regression satisfies the level-set averaging property \eqref{eqn:local_avg}, we wish to minimize the size of each block to maximize the total number of blocks possible. However, since $y$ is binary, we can only represent rational numbers of the form $\frac{a}{b}$, where $a = \sum_{i=u}^v y_i$ denotes the number of ones between $u$ and $v$, and $b = v-u+1$ denotes the total number of samples in that interval. We want to minimize the block size~$b$, subject to the constraints that (1) $0\le a\le b$ and (2) the ratios $\frac{a}{b}$ and $\frac{a'}{b'}$ are distinct across blocks. To achieve the second constraint while making $b$ as small as possible, we insist that the fraction $\frac{a}{b}$ be in standard form, i.e. the greatest common divisor between $a$ and $b$ is 1. We can thus order the potential fitted values of isotonic regression in increasing order of block size:
\begin{equation} \label{eq:farey_series}
    \frac{0}{1}, \frac{1}{1}, \frac{1}{2}, \frac{1}{3}, \frac{2}{3}, \frac{1}{4}, \frac{3}{4}, \frac{1}{5}, 
    \frac{2}{5}, \frac{3}{5}, \frac{4}{5}, \frac{1}{6},
    \frac{5}{6}, \frac{1}{7}, \frac{2}{7}, \frac{3}{7},
    \frac{4}{7}, \frac{5}{7}, \frac{6}{7}, \frac{1}{8}, \frac{3}{8}, \frac{5}{8}, \frac{7}{8}, \frac{1}{9}, \frac{2}{9}, \frac{4}{9}, \frac{5}{9}, \frac{7}{9}, \frac{8}{9}, \ldots
\end{equation}

The number of blocks of a given size $b$ is Euler's totient function $\varphi(b)$ when $b \geq 2$. We thus have the exact expression 

\begin{equation}
\hat{k}(y) 
\leq \sup\left\{1+ \sum_{j=1}^{K-1} \varphi(j)+\ell : (K,\ell)\,\textnormal{ s.t. }\,K\geq 1, \, 0\leq \ell \leq \varphi(K), \, 1+\sum_{j=1}^{K-1} j\varphi(j) + \ell K \le n\right\},\label{eqn:tight_df_bound}  
\end{equation}
where the supremum is taken over integers \(K\) and \(\ell\).

Let $K_0$ refer to the $K$ which attains the supremum above. Without loss of generality, we will assume $K_0\geq 2$, since the $K_0=1$ case is trivial.

We know $\sum_{j=1}^{K} \varphi(j) = \frac{3}{\pi^2}K^2 + O(K\log K)$ \citep[Theorem 330]{hardy1979introduction}. Using this after applying Abel's identity \citep[Theorem 4.2]{apostol2013introduction}, we get that 
\begin{align*}
 \sum_{j=1}^{K_0} j\varphi(j) &= K_0\sum_{j=1}^{K_0} \varphi(j) -\int_1^{K_0} \sum_{1 \leq j \leq t} \varphi(j) dt\\
 &=\frac{3}{\pi^2} K_0^3 + O(K_0^2 \log K_0) - \int_1^{K_0} \left( \frac{3}{\pi^2}t^2 + O(t\log t) \right)dt\\
 &= \frac{2}{\pi^2}K_0^3 + O(K_0^2 \log K_0).
\end{align*}

Since $1+\sum_{j=1}^{K_0-1} j\varphi(j)\le n \leq 1+\sum_{j=1}^{K_0} j\varphi(j)$, we know  
    $n = \frac{2}{\pi^2}K_0^3 + O(K_0^2 \log K_0)$.

Then, by \cref{lemma:inversion_Olog_main_thm}, 
   $K_0 = \left(\frac{\pi^2}{2}\right)^{1/3} n^{1/3} + O(\log n)$.

By plugging this into  $\sup_{y \in \{0,1 \}^n} \hat k(y)= \frac{3}{\pi^2}K_0^2 + O(K_0 \log K_0)$, we know
\begin{equation*}
    \sup_{y \in \{0,1 \}^n} \hat k(y) \leq \frac{3}{(4\pi^2)^{1/3}} n^{2/3} + O(n^{1/3} \log n).
\end{equation*}

\textbf{Lower bound}
Fix $n$ and suppose $1\preceq 2\preceq \dots \preceq n$. Consider $y^\star$ designed to output, after applying isotonic regression, the first $\hat k(y^\star)$ elements in the sequence \cref{eq:farey_series} (corresponding to the Farey series), with each intended fitted value $\frac{a}{b}$ corresponding to a block of $a$ 1s and then $b-a$ 0s. For example, if $n=18$, then  
    \[
\renewcommand{\arraystretch}{1.25}
\begin{array}{r@{\;}c@{\;}*{7}{c@{\;\;}}}
y^\star &=&
\bigl[\underbrace{0} &
      \underbrace{1,0,0,0} &
      \underbrace{1,0,0} &
      \underbrace{1,0} &
      \underbrace{1,1,0} &
      \underbrace{1,1,1,0} &
      \underbrace{1}\bigr]
\\
\mathrm{iso}(y^\star) &=&
\bigl[\overbrace{0} &
      \overbrace{\tfrac14,\tfrac14,\tfrac14,\tfrac14} &
      \overbrace{\tfrac13,\tfrac13,\tfrac13} &
      \overbrace{\tfrac12,\tfrac12} &
      \overbrace{\tfrac23,\tfrac23,\tfrac23} &
      \overbrace{\tfrac34,\tfrac34,\tfrac34,\tfrac34} &
      \overbrace{1}\bigr]
\end{array}
\]
    If $n$ is such that we have leftover samples that cannot represent a new reduced fraction, we will append them as 1s. For example, if $n=19$, then
\[
\renewcommand{\arraystretch}{1.25}
\begin{array}{r@{\;}c@{\;}*{7}{c@{\;\;}}}
y^\star &=&
\bigl[\underbrace{0} &
      \underbrace{1,0,0,0} &
      \underbrace{1,0,0} &
      \underbrace{1,0} &
      \underbrace{1,1,0} &
      \underbrace{1,1,1,0} &
      \underbrace{1,1}\bigr]
\\
\mathrm{iso}(y^\star) &=&
\bigl[\overbrace{0} &
      \overbrace{\tfrac14,\tfrac14,\tfrac14,\tfrac14} &
      \overbrace{\tfrac13,\tfrac13,\tfrac13} &
      \overbrace{\tfrac12,\tfrac12} &
      \overbrace{\tfrac23,\tfrac23,\tfrac23} &
      \overbrace{\tfrac34,\tfrac34,\tfrac34,\tfrac34} &
      \overbrace{1,1}\bigr]
\end{array}
\]

    The fact that $\mathrm{iso}(y^\star)$ outputs the desired sequence follows from the minmax formulation of isotonic projection: $\mathrm{iso}(y)_k = \min _{j \geq k} \max_{i\leq k} \Bar{y}_{i:j}$ \citep[Chapter 1]{RWD88}.

    It follows from the construction of $y^\star$ that
    \begin{equation*}
        \hat{k}(y^\star) 
= \sup\left\{1+ \sum_{j=1}^{K-1} \varphi(j)+\ell : (K,\ell)\,\textnormal{ s.t. }\,K\geq 1, \, 0\leq \ell \leq \varphi(K), \, 1+\sum_{j=1}^{K-1} j\varphi(j) + \ell K \le n\right\},
    \end{equation*}
    matching \cref{eqn:tight_df_bound}.

    Reusing the same steps used to prove the upper bound, we achieve the desired result.
\end{proof}

The result bears similarities to \citet{jarnik1926gitterpunkte}, which shows that the maximum number of integer lattice points a strictly convex curve of length $\ell$ can pass through is $\frac{3}{(4\pi)^{1/3}}\ell^{2/3} + O(\ell^{1/3})$.

%% file: src/calibration_consequences.tex
\section{Consequences for calibration\label{sec:calibration}}

With \cref{thm:main} in hand, we are now able to demonstrate an important statistical property of isotonic regression: it provides calibrated probabilistic forecasts for binary outcomes. Specifically, suppose $X \in [0,1]$ is a probabilistic forecast of whether $Y = 1$. While $X$ may be the outcome of a state-of-the-art neural network, we may not be able to guarantee that $X$ itself is close to $\mathbb E[Y \mid X]$: we hope that an $X\%$ forecast corresponds to a true $X\%$ probability. To improve calibration properties of forecasts, it is common to use isotonic regression to post-process the nominal forecasts $X$, since the monotonicity constraint is natural: larger forecasts should, on average, yield higher true probabilities \citep{zadrozny2002transforming}. Concretely, isotonic calibration entails using isotonic regression to fit a $\hat f$ using fresh data $\{ (X_i, Y_i)\}_{i=1}^n$. For example, if a precipitation model were trained on pre-2020 data, then isotonic calibration may use 2021 to 2025 data to calibrate the initial forecast model. The question remains whether this is guaranteed to improve calibration properties and, if so, how fast. To this end, a natural way to formulate this is assessing whether 
\begin{equation*}
    \Delta(\hat f) :=\mathbb E[|\mathbb E[Y_{n+1}\mid \hat f(X_{n+1}), \hat f] - \hat f(X_{n+1})| \mid \hat f]
\end{equation*}
goes to zero with high probability as $n$ goes to infinity. 

This quantity $\Delta(\hat f)$, called Expected Calibration Error (ECE), is impossible to control even asymptotically for many popular post-processing algorithms, such as Platt scaling \citep{platt1999probabilistic}, since they provide $\hat f$ that are almost surely injective \citep{gupta2020distribution}. To date, it has been an open problem whether isotonic regression is able to control ECE asymptotically, since isotonic regression does not provide an injective function when the degrees of freedom is less than $n$. Therefore, the hardness result of \citet{gupta2020distribution} may not apply, depending on the setting. 

Using our degrees of freedom result, we are able to show that isotonic calibration does in fact control ECE asymptotically in the natural setting where $Y$ is binary. Specifically, we prove that
    $\Delta(\hat f) = O_p(n^{-1/6}\sqrt{\log n})$
when $\hat f$ is an isotonic regressor. This bound on ECE ensures other favorable properties, such as bounds on swap regret, that are of significant importance to decision-makers, from airlines deciding whether to cancel flights given a forecasted chance of a blizzard to doctors selecting treatment plans based on a patient's assessed risk \citep{kleinberg2023u, hu2024calibrationerrordecisionmaking, rossellini2025can}.

\begin{theorem} \label{thm:calibration}
    Suppose $(X_i, Y_i) \overset{iid}{\sim} P$ for $i=1,\dots, n+1$, where $Y \in \{0,1\}$ and $X \in \R$. Suppose $\hat f$ extends the isotonic regression fit on ${(X_i,Y_i)}_{i=1}^n$ 
    by assigning each $x\in\mathbb R$ to a fitted PAVA block, for example by nearest-neighbor interpolation.
 Then for any $\delta \in (0,1)$,
    \[\PP\left(\Delta(\hat{f}) \geq O\left(\sqrt{\frac{\log(n/\delta)}{n^{1/3}}}\right) \right) \leq \delta.\]
\end{theorem}
\begin{proof}
    First we calculate the ECE $\Delta(f)$ for any piecewise constant function $f$, i.e., of the form
    $f(x) = \sum_{j=1}^m a_j\cdot \ind{x\in A_j}$,
    for some $m\geq 1$, distinct values $a_1,\dots,a_m\in[0,1]$, and interval partitioning $\R = A_1\cup\dots \cup A_m$. 
    
    For any interval $A\subseteq\R$, let $P(A) = \PP(X\in A)$, and let $\mu(A) = \EE[Y \mid X\in A]$ if $P(A)>0$ (or, if $P(A)=0$ then arbitrarily we set $\mu(A)=\frac{1}{2}$). Then for each $j$, if $P(A_j)>0$ we have $\EE[Y\mid f(X)=a_j] =\EE[Y\mid X\in A_j] = \mu(A_j)$, and so
    \[\Delta(f) = \EE [|\EE [Y\mid f(X)] - f(X)|] = \sum_{j=1}^m P(A_j) \cdot |a_j - \mu(A_j)|.\]

    Now we return to the data set of size $n$. For any $A\subseteq\R$, define $P_n(A) = \frac{1}{n}\sum_{i=1}^n \ind{X_i\in A}$, 
    and define
    $\mu_n(A) = \frac{\sum_{i=1}^n \ind{X_i\in A}\cdot Y_i}{\sum_{i=1}^n \ind{X_i\in A}}$
    if $P_n(A)>0$ (the sample mean of $Y$ among all data points with $X\in A$), or arbitrarily set $\mu_n(A)=\frac{1}{2}$ if $P_n(A)=0$. For brevity, let $\hat k$ refer to $\hat k(y)$, where $y \in \{0,1 \}^n$ is the sequence of $Y$ values to which we apply PAVA.
    Now recall the construction of the fitted isotonic regression model $\hat{f}$: for some data-dependent partition $\R = A_1 \cup \dots \cup A_{\hat{k}}$, we have $\hat{f}(x) = \mu_n(A_j)$ for all $x\in A_j$.
    Therefore,
    $\Delta(\hat{f}) = \sum_{j=1}^{\hat{k}} P(A_j) \cdot \left|\mu_n(A_j) - \mu(A_j)\right|$.
    
    We can upper bound this as
    \begin{align*}
        \Delta(\hat{f})
        % &= \sum_{j=1}^{\hat{k}} P(A_j) \cdot \left|\mu_n(A_j) - \mu(A_j)\right|\\
        &\leq \sum_{j=1}^{\hat{k}} \left(  \left|P_n(A_j) \cdot \mu_n(A_j) - P(A_j)\cdot \mu_n(A_j)\right|+\left|P_n(A_j) \cdot \mu_n(A_j) - P(A_j)\cdot \mu(A_j)\right|\right)\\
        &\leq \sum_{j=1}^{\hat{k}}\left|P_n(A_j) - P(A_j)\right|+ \sum_{j=1}^{\hat{k}}\left|P_n(A_j) \cdot \mu_n(A_j) - P(A_j)\cdot \mu(A_j)\right| =:\Delta_0 + \Delta_1.
    \end{align*}
These two terms can both be written in the form
       $\Delta_\ell =  \sum_{j=1}^{\hat{k}}  \left|\frac{1}{n}\sum_{i=1}^n \ind{X_i\in A_j} Z^\ell_i - \EE[\ind{X\in A_j}Z^\ell]\right|$
by defining $Z^0=Z^0_i=1$  (for the first term, $\Delta_0$) and $Z^1=Y$, $Z^1_i=Y_i$ (for the second term, $\Delta_1$).
   Next we need the following lemma on empirical processes.
    \begin{lemma}\label{lem:XZ}
        Let $(X,Z)\in\R\times[0,1]$ be drawn from any distribution, and fix $\delta\in (0,1]$. Then, with probability at least $1-\delta$, it holds for all intervals $A\subseteq\R$ that
        \[ \left|\frac{1}{n}\sum_{i=1}^n \ind{X_i\in A} \cdot Z_i  - \EE[\ind{X\in A}Z]\right| \leq \sqrt{\left(\PP(X\in A)+\frac{2}{n}\right)\cdot \frac{4\log(en/\delta)}{n}} + \frac{2\log(en/\delta)/3+2}{n}.
        \]
    \end{lemma}
Applying this lemma, with $Z^0$ or $Z^1$ in place of $Z$ and with $\delta/2$ in place of $\delta$, we then have that for each $\ell=0,1$ it holds with probability at least $1-\delta/2$ that
\begin{align*}
    \sum_{j=1}^{\hat{k}}  \left|\frac{1}{n}\sum_{i=1}^n \ind{X_i\in A_j} Z^\ell_i - \EE[\ind{X\in A_j}Z^\ell]\right| 
    &\leq \sum_{j=1}^{\hat{k}} \left(\sqrt{\left(P(A_j)+\frac{2}{n}\right)\cdot \frac{4\log(2en/\delta)}{n}} + \frac{2\log(2en/\delta)/3+2}{n} \right).
\end{align*}
By Cauchy-Schwarz,
    $\sum_{j=1}^{\hat{k}} \sqrt{(P(A_j)+2/n)} \leq \sqrt{\hat k \sum_{j=1}^{\hat{k}}(P(A_j)+2/n)} \leq \sqrt{\hat k(1+2\hat k/n)} \leq \sqrt{3\hat k}.$
Therefore,
\[\PP\left(\Delta(\hat{f}) \leq 2\left[\sqrt{\frac{12\hat{k}\log(2en/\delta)}{n}} + \frac{2\hat{k}(\log(2en/\delta)/3+2)}{n}\right]\right) \geq 1-\delta.\]
    The proof is then completed by applying \cref{thm:main} to bound $\hat{k}$.
\end{proof}

\begin{remark}
    The approach in \cref{thm:calibration} can in principle be extended to any level-set-averaging algorithm whose preimages lie in a bounded complexity class. For example, unimodal regression's preimages are unions of two intervals. Notably, this means that \cref{thm:calibration} does not directly extend to multivariate isotonic regression: the assumption that $X \in \mathbb R$ was necessary to ensure a total preorder of the covariates, and hence interval preimages.
\end{remark}

\subsection{Comparison with previous calibration results for isotonic regression} First, it has previously been observed \citep{berta2024classifier,rossellini2025can, allen2025sample} that isotonic regression is calibrated in-sample, and therefore has an expected calibration error of zero with respect to the empirical measure. This simply follows from isotonic regression's level-set averaging property. In contrast, \cref{thm:calibration} demonstrates that isotonic regression is calibrated \emph{out-of-sample} for an independent and identically drawn test point. Second, \citet{van2023causal} demonstrate that isotonic regression ensures ECE is $O(n^{-1/3})$ under the assumption that the true conditional mean $\mathbb E[Y\mid X=x]$ is a bounded-variation function. In contrast, our result ensures ECE vanishes at a slower rate, but we require no assumptions on the conditional distribution. Finally, \citet[Theorem 3.1]{allen2025sample} may appear similar to our result, but they essentially bound the distance from calibration \citep{blasiok2023unifying} of isotonic regression, a weaker notion of calibration than ECE. See \citet{rossellini2025can} for additional discussion on these forms of calibration.

%% file: src/discussion.tex
\section{Discussion}\label{sec:discussion}

Our results suggest several open questions for future work.

We believe \cref{thm:main} is loose in the i.i.d.\ setting. In the well-specified isotonic regression setting with homoskedastic Gaussian noise, \citet{meyer2000degrees} show that the expected degrees of freedom is at most $O(n^{1/3})$. Based on our own simulations, we think this upper bound also applies in the binary setting. \citet{allen2025sample} have also made this hypothesis and point to a result \citep[Lemma 3.1]{groeneboom2011vertices} for the Grenander estimator as being potentially relevant.

For \cref{thm:calibration}, we think the result is loose because (1) the deterministic bound on $\hat k$ is loose in i.i.d.\ settings for the reason outlined above, (2) we could not use a bound on $\hat k$ for a proof in the more general $Y \in [0,1]$ setting, since we cannot ensure a sublinear bound on $\hat k$ in general, and (3) in the well-specified isotonic regression setting we would expect a $O(n^{-1/3})$ rate, since consistency error upper-bounds calibration error.

Our approach for \cref{thm:calibration} may be tight under exchangeability, since our bound on $\hat k$ can be achieved almost surely for a specific $P^n$ that mimics our construction of $y^\star$.

%% file: src/acknowledgement.tex
\section*{Acknowledgment}

The authors gratefully acknowledge the National Science Foundation via grant DMS-2023109. 
J.S.\ and R.F.B.\ were partially supported by the Office of Naval Research via grant N00014-24-1-2544. J.S.\ was also partially supported by the Margot and Tom Pritzker Foundation. Z.R.\ was supported by NSF grant DMS-2413135.

%% file: src/appendix.tex
\makeatletter
\@ifpackageloaded{ntheorem}{%
  \newcommand{\keylemmaProofTitle}{of Lemma~\ref{lem:XZ}}%
}{%
  \newcommand{\keylemmaProofTitle}{Proof of Lemma~\ref{lem:XZ}}%
}
\makeatother

\section{Proof of key lemma}

\begin{proof}[\keylemmaProofTitle]
    Let $h:(0,1)\to\R$ be a generalized inverse of the cumulative distribution function of $X$, so that $h$ is monotone nondecreasing, and $h(U)$ is equal in distribution to $X$ when $U\sim\textnormal{Unif}[0,1]$. Without loss of generality we can then replace $X_i$ with $h(U_i)$, where $U_1,\dots,U_n$ are i.i.d.\ uniform random variables (and, $Z_i\mid h(U_i)$ is then drawn from the conditional distribution of $Z\mid X$). For any interval $A\subseteq\R$, since $h$ is monotone, $B=h^{-1}(A)\subseteq(0,1)$ is also an interval. Therefore, we only need to show, for all intervals $B\subseteq[0,1]$, that
            \[ \left|\frac{1}{n}\sum_{i=1}^n \ind{U_i\in B} \cdot Z_i  - \EE[\ind{U\in B}Z]\right| \leq \sqrt{\left(\mathrm{Leb}(B)+\frac{2}{n}\right)\cdot \frac{4\log(en/\delta)}{n}} + \frac{2\log(en/\delta)/3+2}{n},\]
    where we have replaced $(X_i,Z_i)$ with $(h(U_i),Z_i)$ for each $i$, and $(X,Z)$ with $(h(U),Z)$ (so that $X\in A$ is equivalent to $U\in B$, for the appropriate interval $B$). Next, define $B_{jk} = [\frac{j}{n},\frac{k}{n}]$ for each $0\leq j< k \leq n$. For any interval $B\subseteq[0,1]$, by considering the endpoints and rounding to the grid $\{0,\frac{1}{n},\dots,\frac{n-1}{n},1\}$, we can see that we can always find\footnote{Round outward so that $B_{jk} \supseteq B$ if $\frac{1}{n}\sum_{i=1}^n\ind{U_i\in B} \cdot Z_i  - \EE[\ind{U\in B}Z] \geq 0$ and round inward so that $B_{jk} \subseteq B$ otherwise.} some $j,k$ such that
    \[\left|\frac{1}{n}\sum_{i=1}^n \ind{U_i\in B} \cdot Z_i  - \EE[\ind{U\in B}Z] \right|
    \leq \left|\frac{1}{n}\sum_{i=1}^n \ind{U_i\in B_{jk}} \cdot Z_i - \EE[\ind{U\in B_{jk}}Z]\right| + \frac{2}{n},\]
    and,
    \[\left|\mathrm{Leb}(B) - \mathrm{Leb}(B_{jk})\right| \leq \frac{2}{n}.\]
    Next, by the Bernstein inequality, for each $j,k$ it holds with probability $1-\delta$ that
    \[\left|\frac{1}{n}\sum_{i=1}^n \ind{U_i\in B_{jk}} \cdot Z_i - \EE[\ind{U\in B_{jk}}Z]\right|\leq \sqrt{\mathrm{Leb}(B_{jk})\cdot \frac{2\log(2/\delta)}{n}} + \frac{\log(2/\delta)}{3n}. \]
    If we replace $\delta$ with $\frac{2\delta}{n(n+1)}$ to account for all possible grid point combinations, then with probability $1-\delta$ we have simultaneously, for all $0 \leq j<k \leq n$, that
    \[\left|\frac{1}{n}\sum_{i=1}^n \ind{U_i\in B_{jk}} \cdot Z_i - \EE[\ind{U\in B_{jk}}Z]\right|\leq \sqrt{\mathrm{Leb}(B_{jk})\cdot \frac{2\log(n(n+1)/\delta)}{n}} + \frac{\log(n(n+1)/\delta)}{3n}. \]
    Putting it all together, we have with probability $1-\delta$
    \begin{align*}
        \left|\frac{1}{n}\sum_{i=1}^n \ind{U_i\in B} \cdot Z_i  - \EE[\ind{U\in B}Z] \right| &\leq \sqrt{\mathrm{Leb}(B_{jk})\cdot \frac{2\log(n(n+1)/\delta)}{n}} + \frac{\log(n(n+1)/\delta)/3+2}{n}\\
        &\leq \sqrt{\left(\mathrm{Leb}(B)+\frac{2}{n}\right)\cdot \frac{2\log(n(n+1)/\delta)}{n}} + \frac{\log(n(n+1)/\delta)/3+2}{n}.
    \end{align*}
    For $\delta \in (0,1)$ and $n\geq 1$, we have $\log \frac{n(n+1)}{\delta} \leq \log \frac{2 n^2}{\delta} \leq 2 \log \frac{e n}{\delta}$.
\end{proof}

\section{The constants on the lower order terms for \cref{thm:main}}

\begin{proposition}
    Let $T: [0,1]^n \to [0,1]^n$ be any level-set averaging algorithm \eqref{eqn:local_avg}. Let $\hat k(y)$ be the number of unique values outputted by $T(y)$. Then, $\forall y \in \{0,1\}^n$, we have
\begin{equation*}
    \hat{k}(y)-1 \leq \frac{3 n^{2 / 3}}{(2 \pi)^{2 / 3}}+\frac{11 \pi^{2 / 3} n^{1/3} \log (n)}{2^{1/3}}+\frac{3 \cdot 2^{2 / 3} n^{1/3}}{\pi^{4 / 3}}+\frac{39}{4} \pi^2 \log ^2(n)+11 \log (n)+\frac{3}{\pi^2}
\end{equation*}
\end{proposition}
\begin{proof}
    Under the same logic in \cref{thm:main}, we know that
    \[
\sup_{y\in \{0,1\}^n}\hat{k}(y) 
\leq \sup\left\{1+ \sum_{j=1}^{K-1} \varphi(j)+\ell : (K,\ell)\,\textnormal{ s.t. }\,K\geq 1, \, 0\leq \ell \leq \varphi(K), \, 1+\sum_{j=1}^{K-1} j\varphi(j) + \ell K \le n\right\}.
\]

Let $(K_0,\ell_0)$ refer to the pair which attains the supremum above.

When $K_0\geq2$, by \cref{lemma:euler_totient_moment} we know $\left| \sum_{j=1}^{K_0} \varphi(j) - \frac{3}{\pi^2}K_0^2\right|\leq  2K_0\log K_0 $. Using this after applying Abel's identity \citep[Theorem 4.2]{apostol2013introduction}, we get that 
\begin{align*}
 \sum_{j=1}^{K_0} j\varphi(j) &= K_0\sum_{j=1}^{K_0} \varphi(j) -\int_1^{K_0} \sum_{1 \leq j \leq t} \varphi(j) dt\\
 &\geq \frac{3}{\pi^2} K_0^3 - 2K_0^2\log K_0 - \int_2^{K_0} \left(\frac{3}{\pi^2}t^2 + 2 t \log t \right)dt -\int_1^{2} \sum_{1 \leq j \leq t} \varphi(j) dt\\
 &= \frac{3}{\pi^2} K_0^3 - 2K_0^2\log K_0 - \frac{1}{\pi^2} K_0^3-K_0^2\log K_0 + K_0^2/2 +\frac{8}{\pi^2} + 4\log 2 - 2 - 1\\
 &\geq  \frac{2}{\pi^2} K_0^3 - 3K_0^2 \log K_0.
\end{align*}
%https://www.wolframalpha.com/input?i=-%5Cint_2%5EK+%5Cleft%28%5Cfrac%7B3%7D%7B%5Cpi%5E2%7Dt%5E2+%2B+2+t+%5Clog+t+%5Cright%29dt&assumption=%7B%22F%22%2C+%22Integral%22%2C+%22rangestart%22%7D+-%3E%222%22&assumption=%7B%22C%22%2C+%22integral%22%7D+-%3E+%7B%22Calculator%22%2C+%22dflt%22%7D&assumption=%7B%22F%22%2C+%22Integral%22%2C+%22integrand%22%7D+-%3E%22%283%29%2F%28%5Cpi%5E2%29+t%5E2+%2B+2+t+%5Clog+t%22&assumption=%7B%22F%22%2C+%22Integral%22%2C+%22rangeend%22%7D+-%3E%22K%22

Since by construction $\sum_{j=1}^{K_0-1} j\varphi(j)\le n-1 \leq n$ and $\log (K_0-1) \leq \log n$, we know  
\begin{equation*}
    (K_0-1)^3\leq (\pi^2 /2 )n + (3\pi^2 /2 )(K_0-1)^2 \log n.
\end{equation*}

In general, by \cref{lemma:elementary_cubic_ineq}, we have $x^3 \leq a + bx^2 \implies x \leq a^{1/3} + b$ for $a,b\geq 0$, so 
\begin{equation*}
    K_0-1\leq (\pi^2 /2 )^{1/3}n^{1/3} + (3\pi^2 /2 )\log n.
\end{equation*}

Since $\sup_{y\in \{0,1\}^n}\hat{k}(y)-1 \leq \sum_{j=1}^{K_0} \varphi(j) \leq \frac{3}{\pi^2}K_0^2 + 2K_0 \log K_0$, we get
\begin{equation*}
    \sup_{y\in \{0,1\}^n}\hat{k}(y)-1 \leq \frac{3 n^{2 / 3}}{(2 \pi)^{2 / 3}}+\frac{11 \pi^{2 / 3} n^{1/3} \log (n)}{2^{1/3}}+\frac{3 \cdot 2^{2 / 3} n^{1/3}}{\pi^{4 / 3}}+\frac{39}{4} \pi^2 \log ^2(n)+11 \log (n)+\frac{3}{\pi^2}
\end{equation*}
by once again using $\log K_0 \leq \log n$.
% Relevant Wolfram-Alpha query: https://www.wolframalpha.com/input?i=%5Cfrac%7B3%7D%7B%5Cpi%5E2%7DK_0%5E2+%2B+2K_0+%5Clog+n+where+K_0%3D+%28%5Cpi%5E2+%2F2+%29%5E%7B1%2F3%7Dn%5E%7B1%2F3%7D+%2B+%283%5Cpi%5E2+%2F2+%29%5Clog+n%2B1
\end{proof}

\section{Other supporting lemmas}

We can provide a constant for the $O(K\log K)$ term in $\sum_{j=1}^K \varphi(j) = \frac{3}{\pi^2}K^2 + O(K\log K)$ \citep[Theorem 330]{hardy1979introduction} under the following lemma.
\begin{lemma} \label{lemma:euler_totient_moment}
    Let $\varphi$ be the Euler totient function. Then, $\forall K\geq 2$,
    \begin{equation*}
        \left| \sum_{j=1}^K \varphi(j) - \frac{3}{\pi^2}K^2\right|\leq  2K\log K 
    \end{equation*}
\end{lemma}
\begin{proof}
    Following \citet[Theorem 330]{hardy1979introduction}, we know that
    \begin{equation*}
        \sum_{j=1}^K \varphi(j) = \frac{1}{2} \sum_{j=1}^K \mu(j)\left(\left\lfloor \frac{K}{j} \right\rfloor^2 + \left\lfloor \frac{K}{j} \right\rfloor\right),
    \end{equation*}
    where $\mu: \mathbb N \to \{-1,0,1 \}$ is the Möbius function.

    Defining $r_j := K/j - \left\lfloor \frac{K}{j} \right\rfloor$, we get
    \begin{equation*}
        \frac{1}{2}\sum_{j=1}^K \mu(j)\left(\left( \frac{K}{j} -r_j\right)^2 + \frac{K}{j} -r_j\right) = \frac{1}{2}\sum_{j=1}^K \mu(j)\left(\frac{K^2}{j^2} -2r_j\frac{K}{j} + r_j^2 + \frac{K}{j} -r_j\right).
    \end{equation*}
    With the first term, we see that
    \begin{equation*}
        \frac{K^2}{2}\sum_{j=1}^K \mu(j)\frac{1}{j^2} = \frac{K^2}{2}\left(\frac{1}{\zeta(2)} - \sum_{j=K+1}^\infty \frac{\mu(j)}{j^2}\right) = \frac{3}{\pi^2}K^2 - \frac{K^2}{2} \sum_{j=K+1}^\infty \frac{\mu(j)}{j^2}, 
    \end{equation*}
    where $\zeta(\cdot)$ is the Riemann zeta function.

    Thus,
    \begin{equation*}
        \left| \sum_{j=1}^K \varphi(j) - \frac{3}{\pi^2}K^2\right| \leq \left|- \frac{K^2}{2} \sum_{j=K+1}^\infty \frac{\mu(j)}{j^2} \right| +  \left|\frac{K}{2} \sum_{j=1}^K\mu(j) \frac{1-2r_j}{j} \right| + \left|\frac{1}{2} \sum_{j=1}^K\mu(j)(r_j^2 - r_j)\right|.
    \end{equation*}
    
    Dealing with each term, we first see that
    \begin{equation*}
        \left|- \frac{K^2}{2} \sum_{j=K+1}^\infty \frac{\mu(j)}{j^2} \right| \leq \frac{K^2}{2} \int_{K}^\infty \frac{1}{x^2}dx = \frac{K}{2}
    \end{equation*}
    because $|\mu(j)| \leq 1$.

    Since $r_j \in [0,1)$ and $K\geq 2$, we know that
    \begin{equation*}
        \left|\frac{K}{2} \sum_{j=1}^K\mu(j) \frac{1-2r_j}{j} \right| \leq \frac{K}{2} \sum_{j=1}^K \frac{1}{j} \leq \frac{K}{2}(\log K + \gamma + \frac{1}{2K}).
    \end{equation*}
    where $\gamma$ is the Euler-Mascheroni constant and $\gamma  < 0.57722$.

    Similarly,
    \begin{equation*}
        \left|\frac{1}{2} \sum_{j=1}^K\mu(j)(r_j^2 - r_j)\right| \leq \frac{1}{2} \sum_{j=1}^K \frac{1}{4} = \frac{K}{8}.
    \end{equation*}

    Since $K\geq2$ implies $\log K \geq \log 2$, we have that
    \begin{align*}
        K\left(\frac{1}{2} + 0.5\log K + 0.5\gamma + \frac{1}{4K} + \frac{1}{8}\right) &\leq K \log K \left(\frac{1}{2 \log 2} + 0.5 + \frac{0.5\gamma}{\log 2} + \frac{1}{4K\log 2} + \frac{1}{8\log 2}\right)\\
        &\leq K \log K \left(\frac{1}{2 \log 2} + 0.5 + \frac{0.5\gamma}{\log 2} + \frac{1}{4\log 2}\right)\\
        &\leq 1.999K\log K
    \end{align*}
\end{proof}

\begin{lemma} \label{lemma:elementary_cubic_ineq}
    If $x^3 \leq a + bx^2$ with $a,b\geq 0$, then
    \begin{equation*}
        x \leq a^{1/3}+b.
    \end{equation*}
\end{lemma}
\begin{proof}
    Suppose $x > a^{1/3}+b$. Then,
    \begin{equation*}
        x^3 - bx^2 = x^2(x-b) > x^2 a^{1/3} \geq \left(a^{1/3}\right)^2 a^{1/3} = a.
    \end{equation*}
    In short, $x^3 > a+bx^2$. This is a contradiction.
\end{proof}

\begin{lemma} \label{lemma:elementary_cubic_ineq_for_lower_bound}
    Let $a,b,c>0$. Let $x_0 := (c/a)^{1/3}$. Suppose $x> 0$ satisfies
    \begin{equation*}
        c \leq ax^3 + bx^2.
    \end{equation*}
    Then,
    \begin{equation*}
        x \geq x_0 - b/a.
    \end{equation*}
\end{lemma}

\begin{proof}
    If $x \geq x_0$, then the result holds trivially.

    If $x < x_0$, then
    \begin{equation*}
        a(x_0^3-x^3) = a(x_0 - x)(x_0^2 + x_0x + x^2) \geq a(x_0-x)x^2.
    \end{equation*}

    Observe that, since $ax_0^3 = c$
    \begin{equation*}
        ax^3 = ax_0^3 - a(x_0^3 - x^3) \leq c - a(x_0-x)x^2 \leq ax^3 + bx^2 - a(x_0-x)x^2,
    \end{equation*}
    where the last inequality was by supposition.

    This in turn implies
    \begin{equation*}
        a(x_0 - x)x^2 \leq bx^2,
    \end{equation*}
    so in conclusion $x_0 - x \leq b/a$.
\end{proof}

\begin{lemma} \label{lemma:inversion_Olog_main_thm}    
    Suppose $1 \leq K \leq n$, $d>0$, and 
    \begin{equation*}
        dK^3 = n + O(K^2 \log  K).
    \end{equation*}
    Then,
    \begin{equation*}
        d^{1/3} K - n^{1/3} = O(\log n). 
    \end{equation*}
\end{lemma}

\begin{proof}
    First, we know $\exists C >0$ such that
    \begin{equation*}
        K^3 \leq  \frac{1}{d}n + \frac{C}{d}K^2 \log K \leq \frac{1}{d}n + \frac{C}{d}K^2 \log n.
    \end{equation*}
    Then, by \cref{lemma:elementary_cubic_ineq},
    \begin{equation*}
        K \leq (n/d)^{1/3} + \frac{C}{d}\log n.
    \end{equation*}

    Second, we know $\exists B > 0$ such that
    \begin{equation*}
        dK^3 \geq  n - BK^2 \log K \geq n - BK^2 \log n.
    \end{equation*}
    Then, by \cref{lemma:elementary_cubic_ineq_for_lower_bound}, 
    \begin{equation*}
        K \geq (n/d)^{1/3} - \frac{B}d\log n
    \end{equation*}
\end{proof}

%% file: src/error_term_certificate.tex
\section{Numerical bound on the error term $\delta_n$}

Define 
\begin{align*}
	f(n) &= \sup\left\{1+\sum_{j=1}^{K-1}\varphi(j) + \ell : (\ell, K) \text{ s.t. } K\ge 1, ~0\le\ell\le \varphi(K),~ 1+\sum_{j=1}^{K-1}j\varphi(j) + \ell K\le n\right\}.
\end{align*}
In this section, we bound the difference
\begin{align}
	\delta_n &= f(n) - cn^{2/3}, \quad \text{ where }\quad c = \frac{3}{(4\pi^2)^{1/3}}.
\end{align}

\begin{proposition}\label{prop:error_term_certificate}
	$\delta_n \le \alpha n^{1/6}$ for all $n\le 2\cdot 10^{23}$, where $\alpha=0.63$.
\end{proposition}

\begin{proof} In the definition of $f(n)$, $\ell$ is restricted to integers, but if we allow it to be arbitrary scalars in the interval $[0, \varphi(K)]$, we can write the supremum more explicitly. Define
\[
 n_K=1+\sum_{j\leq K}j\varphi(j),\qquad
 B_K=1+\sum_{j\leq K}\varphi(j),\qquad R_K=B_K-cn_K^{2/3}.
\]
and let $K=K_n$ be such that $n_{K-1}\le n\le n_K$. Let $\ell$ satisfy $n_{K-1} + \ell K = n$. Then
\[
f(n)\le B_{K-1}+\ell.
\]
Setting \(\lambda_n=(n-n_{K-1})/(n_K-n_{K-1})\), we have
\begin{align}\label{eq:interpolation-f}
 f(n)\le(1-\lambda_n)B_{K-1}+\lambda_n B_K.
\end{align}
Since $n\mapsto n^{2/3}$ is concave and $n = (1-\lambda_n)n_{K-1}+\lambda_n n_K$, we have
\[
\delta_n \le (1-\lambda_n)B_{K-1}+\lambda_n B_K - cn^{2/3}\le (1-\lambda_n)R_{K-1}+\lambda_n R_K.
\]
Numerically, we check that $R_K\le \alpha n_K^{1/6}$ for all $K\le 10^8$ (see the Python code that follows the proof). Since $n\mapsto \alpha n^{1/6}$ is concave and $n = (1-\lambda_n)n_{K-1}+\lambda_n n_K$, we have
\[
\delta_n 
\le (1-\lambda_n)\alpha n_{K-1}^{1/6}+\lambda_n \alpha n_K^{1/6}
\le \alpha n^{1/6}.
\]
\end{proof}

\begin{verbatim}
import numpy as np

def totients(K):
    """Return totient sequence phi(0)=0, phi(1), ..., phi(K)."""
    phi = np.arange(K + 1, dtype=np.uint64)

    for p in range(2, K + 1):
        if phi[p] == p:
            phi[p::p] -= phi[p::p] // p

    return phi

K = 100_000_000
Ks = np.arange(K + 1)

phi = totients(K)

nK = 1+np.cumsum(Ks*phi)
BK = 1+np.cumsum(phi)
c = 3/(4*np.pi**2)**(1/3)
RK = BK - c*nK**(2/3)

print(np.log10(nK[-1]), np.max(RK/nK**(1/6)))
\end{verbatim}

%% file: main.bbl
\begin{thebibliography}{}

\bibitem[Allen et~al., 2025]{allen2025sample}
Allen, S., Gavrilopoulos, G., Henzi, A., Kleger, G.-R., and Ziegel, J. (2025).
\newblock In-sample calibration yields conformal calibration guarantees.
\newblock {\em arXiv preprint arXiv:2503.03841}.

\bibitem[Apostol, 2013]{apostol2013introduction}
Apostol, T.~M. (2013).
\newblock {\em Introduction to analytic number theory}.
\newblock Springer Science \& Business Media.

\bibitem[Berta et~al., 2024]{berta2024classifier}
Berta, E., Bach, F., and Jordan, M. (2024).
\newblock Classifier calibration with roc-regularized isotonic regression.
\newblock In {\em International Conference on Artificial Intelligence and Statistics}, pages 1972--1980. PMLR.

\bibitem[B{\l}asiok et~al., 2023]{blasiok2023unifying}
B{\l}asiok, J., Gopalan, P., Hu, L., and Nakkiran, P. (2023).
\newblock A unifying theory of distance from calibration.
\newblock In {\em Proceedings of the 55th Annual ACM Symposium on Theory of Computing}, pages 1727--1740.

\bibitem[Dimitriadis et~al., 2023]{dimitriadis2023honest}
Dimitriadis, T., D{\"u}mbgen, L., Henzi, A., Puke, M., and Ziegel, J. (2023).
\newblock Honest calibration assessment for binary outcome predictions.
\newblock {\em Biometrika}, 110(3):663--680.

\bibitem[Fris{\'e}n, 1986]{frisen1986unimodal}
Fris{\'e}n, M. (1986).
\newblock Unimodal regression.
\newblock {\em Journal of the Royal Statistical Society Series D: The Statistician}, 35(4):479--485.

\bibitem[Ghosh et~al., 2004]{repec:bep:mchbio:1037}
Ghosh, D., Banerjee, M., and Biswas, P. (2004).
\newblock Binary isotonic regression procedures, with application to cancer biomarkers.
\newblock The University of Michigan Department of Biostatistics Working Paper Series 1037, Berkeley Electronic Press.

\bibitem[Groeneboom, 2011]{groeneboom2011vertices}
Groeneboom, P. (2011).
\newblock Vertices of the least concave majorant of brownian motion with parabolic drift.

\bibitem[Gupta et~al., 2020]{gupta2020distribution}
Gupta, C., Podkopaev, A., and Ramdas, A. (2020).
\newblock Distribution-free binary classification: prediction sets, confidence intervals and calibration.
\newblock {\em Advances in Neural Information Processing Systems}, 33:3711--3723.

\bibitem[Hardy and Wright, 1979]{hardy1979introduction}
Hardy, G.~H. and Wright, E.~M. (1979).
\newblock {\em An introduction to the theory of numbers}.
\newblock Oxford university press.

\bibitem[Hu and Wu, 2024]{hu2024calibrationerrordecisionmaking}
Hu, L. and Wu, Y. (2024).
\newblock Calibration error for decision making.

\bibitem[Jarn{\'\i}k, 1926]{jarnik1926gitterpunkte}
Jarn{\'\i}k, V. (1926).
\newblock {\"U}ber die gitterpunkte auf konvexen kurven.

\bibitem[Kleinberg et~al., 2023]{kleinberg2023u}
Kleinberg, B., Leme, R.~P., Schneider, J., and Teng, Y. (2023).
\newblock U-calibration: Forecasting for an unknown agent.
\newblock In {\em The Thirty Sixth Annual Conference on Learning Theory}, pages 5143--5145. PMLR.

\bibitem[Lee, 1983]{lee1983min}
Lee, C.-I.~C. (1983).
\newblock The min-max algorithm and isotonic regression.
\newblock {\em The Annals of Statistics}, pages 467--477.

\bibitem[Meyer and Woodroofe, 2000]{meyer2000degrees}
Meyer, M. and Woodroofe, M. (2000).
\newblock On the degrees of freedom in shape-restricted regression.
\newblock {\em The annals of Statistics}, 28(4):1083--1104.

\bibitem[Platt et~al., 1999]{platt1999probabilistic}
Platt, J. et~al. (1999).
\newblock Probabilistic outputs for support vector machines and comparisons to regularized likelihood methods.
\newblock {\em Advances in large margin classifiers}, 10(3):61--74.

\bibitem[Robertson et~al., 1988]{RWD88}
Robertson, T., Wright, F.~T., and Dykstra, R.~L. (1988).
\newblock {\em Order Restricted Statistical Inference}.
\newblock Wiley, Chichester.

\bibitem[Rossellini et~al., 2025]{rossellini2025can}
Rossellini, R., Soloff, J.~A., Barber, R.~F., Ren, Z., and Willett, R. (2025).
\newblock Can a calibration metric be both testable and actionable?
\newblock {\em arXiv preprint arXiv:2502.19851}.

\bibitem[Soundararajan, 2009]{Soundararajan+2009+141+152}
Soundararajan, K. (2009).
\newblock Partial sums of the möbius function.
\newblock {\em Journal für die reine und angewandte Mathematik}, 2009(631):141--152.

\bibitem[Sparre~Andersen, 1954]{andersen1954fluctuations}
Sparre~Andersen, E. (1954).
\newblock On the fluctuations of sums of random variables ii.
\newblock {\em Mathematica Scandinavica}, pages 195--223.

\bibitem[Van Der~Laan et~al., 2023]{van2023causal}
Van Der~Laan, L., Ulloa-P{\'e}rez, E., Carone, M., and Luedtke, A. (2023).
\newblock Causal isotonic calibration for heterogeneous treatment effects.
\newblock In {\em International Conference on Machine Learning}, pages 34831--34854. PMLR.

\bibitem[Zadrozny and Elkan, 2001]{zadrozny2001obtaining}
Zadrozny, B. and Elkan, C. (2001).
\newblock Obtaining calibrated probability estimates from decision trees and naive bayesian classifiers.
\newblock In {\em Icml}, volume~1, pages 609--616.

\bibitem[Zadrozny and Elkan, 2002]{zadrozny2002transforming}
Zadrozny, B. and Elkan, C. (2002).
\newblock Transforming classifier scores into accurate multiclass probability estimates.
\newblock In {\em Proceedings of the eighth ACM SIGKDD international conference on Knowledge discovery and data mining}, pages 694--699.

\end{thebibliography}
